\documentclass[letterpaper, 10 pt, conference]{ieeeconf}  % Comment this line out if you need a4paper

\IEEEoverridecommandlockouts                              % This command is only needed if 
\usepackage{amsmath}
\usepackage{amssymb}
\usepackage{amsthm}
\usepackage{xcolor}
\usepackage{algorithm}
\usepackage{algpseudocode}
\usepackage{algorithmicx}
\usepackage{bbm}
\usepackage{graphicx}
\usepackage{comment}
\usepackage{subcaption}
\usepackage{dsfont}

\newtheorem{remark}{Remark}
\newtheorem{assumption}{Assumption}
\newtheorem{theorem}{Theorem}

\newtheorem{lemma}{Lemma}

\usepackage{etoolbox}
\makeatletter
\patchcmd{\@addpunct}{:}{\space}{}{}
\makeatother

\makeatletter
\renewenvironment{proof}[1][\proofname]{\par
  \pushQED{\qed}%
  \normalfont \topsep6\p@\@plus6\p@\relax
  \trivlist
  \item[\hskip\labelsep\itshape #1.\@addpunct{\space}]\ignorespaces
}{%
  \popQED\endtrivlist\@endpefalse
}
\makeatother

\title{\LARGE \bf
Exploiting Adjacent Similarity in Multi-Armed Bandit Tasks via Transfer of Reward Samples
}

\author{NR Rahul$^{1}$ and Vaibhav Katewa$^{2}$% <-this % stops a space
\thanks{NR Rahul is with the Department of Electrical  Communication Engineering (ECE) at the Indian Institute of Science, Bengaluru, India. Email: {\tt\small rahulnr@iisc.ac.in}.}
\thanks{Vaibhav Katewa is with the Robert Bosch Center for Cyber-Physical Systems and the Department of ECE at the Indian Institute of Science, Bengaluru, India. Email:{\tt\small vkatewa@iisc.ac.in}}%
}

\begin{document}

\maketitle
\thispagestyle{empty}
\pagestyle{empty}

%%%%%%%%%%%%%%%%%%%%%%%%%%%%%%%%%%%%%%%%%%%%%%%%%%%%%%%%%%%%%%%%%%%%%%%%%%%%%%%%
\begin{abstract}

We consider a sequential multi-task problem, where each task is modeled as the stochastic multi-armed bandit with $K$ arms. We assume the bandit tasks are adjacently similar in the sense that the difference between the mean rewards of the arms for any two consecutive tasks is bounded by a parameter. We propose two algorithms (one assumes the parameter is known while the other does not) based on UCB to transfer reward samples from preceding tasks to improve the overall regret across all tasks. Our analysis shows that transferring samples reduces the regret as compared to the case of no transfer. We provide empirical results for our algorithms, which show performance improvement over the standard UCB algorithm without transfer and a naive transfer algorithm.

\end{abstract}

%%%%%%%%%%%%%%%%%%%%%%%%%%%%%%%%%%%%%%%%%%%%%%%%%%%%%%%%%%%%%%%%%%%%%%%%%%%%%%%%
\section{INTRODUCTION}

In sequential multi-task settings, an agent encounters a sequence of tasks to be solved. The agent can transfer information from previously solved tasks to new and similar tasks to help improve performance of the new task\cite{pan2009survey}\cite{zhuang2020comprehensive}. In the context of multi-armed bandits \cite{bubeck2012regret}\cite{lattimore2020bandit}, the information from one bandit is used to make decisions in another similar bandit task\cite{bouneffouf2020survey}. This is particularly useful in scenarios such as the user cold start problem \cite{silva2022multi} in recommender systems, where good initial recommendations are made by using information gathered from similar users. Similarly, in reinforcement learning, the learned policies/model from one task is used in another similar task to help speed up learning\cite{taylor2009transfer}\cite{zhu2023transfer}. Reusing information also helps to address the problem of data efficiency in reinforcement learning.

In this paper, we consider a sequential multi-task setting, where the agent interacts with each task sequentially, one after the other. Each task is modeled by a stochastic multi-armed bandit problem, where the agent interacts by pulling one of the arms at any given time and, in return, gets a random reward. We assume the tasks are adjacently similar and introduce a parameter $\epsilon$ in Section \ref{prob_sta} to capture this similarity between tasks. The parameter $\epsilon$ captures many interesting scenarios like the similarity between different users based on region, age, gender, etc, and the changing user preferences in recommender systems, changing market trends in online advertising, etc. The goal is to use the information from the previously solved tasks in order to improve the performance in the current task, therefore leading to overall performance improvement. This is achieved by reusing/transferring reward samples from previously encountered tasks to the current task. Our algorithm is inspired by \cite{10590903}, which is based on Upper Confidence Bound (UCB) algorithm \cite{auer2002finite}.

\textbf{Related Work.} 
Several works in transfer learning for bandits have focused on linear\cite{cella2020meta,cella2023multi} or contextual bandits \cite{cai2024transfer,deshmukh2017multi}, where an explicit form of the reward function is assumed. This makes the problem simpler and allows the derivation of mathematical results. In contrast, we consider transfer learning in the most generic case of stochastic multi-armed bandits, where no such assumption of the reward function is made. In [6], the authors study sequential transfer in stochastic multi-armed bandits. However, they consider a fixed number of MAB tasks. In contrast, we study sequential transfer in stochastic multi-armed bandits for infinite tasks. The authors in \cite{soare2014multi} have extended this framework to infinite linear bandits tasks that are close in $l_2$ distance. In our previous paper on transfer in MAB \cite{10590903}, we considered the notion of universal similarity, where all tasks are similar. In contrast, we study adjacent similarity in this paper.
%which differs from our work in the similarity notion of tasks considered. In particular, we consider the adjacent tasks to be similar, unlike the universal similarity in \cite{rahul2024transfer}, where all tasks are similar. 
Therefore, we transfer reward samples from the preceding task and not from all previous tasks. Additionally, the number of reward samples to transfer is controlled through a novel parameter, which effectively transfers more samples if the tasks are close and fewer samples if the tasks are not. On the other hand, in \cite{10590903}, all previous reward samples are transferred. Thus, this paper generalizes the setting of \cite{10590903} in a non-trivial manner.
% Furthermore, paper \cite{qin2021non,cella2023multi} considers sharing representations across MAB tasks in linear bandits. 

Another similar set of problems are non-stationary bandits \cite{garivier2011upper,liu2018change}, which are different from our setting in the sense that the task-switching instants are unknown. Although, in our setting, the task-switching instants are known, we provide transfer algorithms to improve the performance over the no-transfer UCB-based algorithm (NT-UCB). Note that the algorithms in the literature of non-stationary bandits use NT-UCB with known switching task instants as the oracle algorithm. Therefore, we believe our approach of transferring reward samples to non-stationary bandits will achieve better performance (which we defer to future work). 

% Sequential transfer in multi-armed bandits has been studied by authors in \cite{lazaric2013sequential} for a fixed number of tasks. The authors in \cite{soare2014multi} have extended this framework to linear bandits with tasks that are close in $l_2$ distance. Another related work \cite{r2024transfer}, uses $l_{\inf}$ distance to capture the similarity between tasks for the sequential stochastic bandits. However, these works do not fully capture the sequential dependence between consecutive tasks. In our work, we consider consecutive tasks to be related, effectively capturing the sequential dependency among tasks. In addition, we learn the closeness of sequential tasks from data. Another related work \cite{qin2022non} is representation transfer in sequential linear multi-armed bandits.   

% Our work is inspired from \cite{r2024transfer} and differs in several key aspects,\\
% (i) We consider a different notion of similarity between tasks, which better captures the sequential dependency.\\ 
% (ii) The algorithm proposed in [\cite{r2024transfer}] cannot be directly extended to our setting.\\
% (iii) We propose algorithms for scenarios where the task similarity parameter is both known and unknown.

\textbf{Main Contributions.} The main contributions of the paper are:
\begin{enumerate}
\item We propose Tr-UCB algorithm to transfer information using the reward samples from the preceeding task to the current task in a sequential multi-task bandit setting. We extend Tr-UCB to Tr-UCB2 algorithm to handle the case of unknown parameter $\epsilon$.
\item We provide the regret analysis for Tr-UCB and Tr-UCB2 and show that there is no negative transfer. Our regret upper bound clearly captures the performance improvement due to transfer.
\item We provide empirical evaluation of the algorithms Tr-UCB and Tr-UCB2 and show effectiveness of transferring information from previous tasks.
\end{enumerate}

\textbf{Notations:} $\mathds{1}\{E\}$ denotes the indicator function whose value is $1$ if the event (condition) $E$ is true, and $0$ otherwise. Similarly, for $n$ events $E_1$, $E_2$, $\cdots$, $E_n$, where $n\in\mathbb{N}$, we define $\mathds{1}\{E_1, E_2, \cdots, E_n\}$ as the indicator function whose value is $1$ if all the events are true, and $0$ otherwise. Further, let $\emptyset$ denote the null set, and let $[l]$ denote the set $\{1, 2, \cdots, l\}$ for some $l\in \mathbb{N}$.

\section{Preliminaries and Problem Statement}

\label{prob_sta}

We consider a sequential multi-task problem, where each task is modeled as a stochastic multi-armed bandit with $K$ arms. Let $J$ denote the total number of tasks, and $n_j$ denote the task length/total steps in task $j$. Further, let the total number of steps in the $J$ tasks be denoted by $T$ and is given by $T = \sum\limits_{j=1}^{J}n_j$. In task $j$, at each time step $t$ (denotes the number of steps from the beginning of the task $j$), the agent makes a decision denoted by $I_t^j\in[K]$ to pull one of the $K$ arms, and in turn, receives a random reward $r_{I_t^j}\in[0,1]$. Let $\mathbf{r}_{t}^{j} = \{r_{I_{1}^j}, r_{I_{2}^j},..., r_{I_{t}^j}\}$ denote the corresponding rewards received from steps $1$ to $t$ in task $j$. The reward samples are independent across time and across arms, and their probability distributions are unknown. Let $\mu_{k}^{j}$ be the mean reward of arm $k$ in task $j$. We define $k^{j}_{*}$ and $\mu^{j}_{*}$ to be an optimal arm in task $j$ and its mean reward, respectively, and are given by,
\begin{align*}
k^{j}_{*} \in \mathcal{A}^j = \underset{k\in [K]}{\arg \max} \{\mu_{k}^{j}\}\hspace{8pt} \text{and} \hspace{8pt} \mu^{j}_{*} = \max\limits_{k\in [K]} \{\mu_{k}^{j}\}.
\end{align*}
Define $\Delta_{k}^{j} = \mu_{*}^{j}-\mu_{k}^{j}>0$ as the sub-optimality gap of arm $k \notin \mathcal{A}^j$ in task $j$. In our setting, the agent encounters a sequence of multi-armed bandit tasks (refer to Figure \ref{fig:SeqMultiTaskSett}). We assume the tasks are adjacently similar in the sense that the mean rewards of consecutive tasks do not change considerably. The following assumption captures the similarity between any two consecutive tasks.

\begin{assumption}
\label{assump1}
 We assume that $|\mu^{j}_k - \mu^{j+1}_k| \leq \epsilon_k $ for all $j\in [J-1]$, and the parameter $\epsilon_k \in[0,1), \forall k\in[K]$.
\end{assumption}
This assumption implies that for each arm $k\in[K]$, the mean rewards between any two consecutive tasks do not differ by more than $\epsilon_k$. One application where this assumption is relevant is online advertising and recommender systems, wherein user preferences do not change drastically over time. Note that for any given task, Assumption \ref{assump1} can be leveraged to have a better inference about the optimal mean reward (and its corresponding optimal arm) of that particular task by using additional reward samples from the previous similar tasks. 

The goal of the agent in the sequential multi-armed bandit setting in any given task $j$ and time $t$ is to make decisions $I_t^j$ based on the reward samples $\{\{\mathbf{r}_{n_l}^{l}\}_{l=1}^{j-1}, \mathbf{r}_{t-1}^{j}\}$ to maximize the expected total reward over all the bandit tasks. This is captured in terms of the total pseudo-regret as
\begin{align} \label{eq_regret}
R_J = \sum\limits_{j=1}^{J} R_{n_j} = \sum\limits_{j=1}^{J}\left[n_j\mu_{*}^{j}-\mathbb{E}\left[\sum\limits_{t=1}^{n_j}\mu^j_{I_t^j}\right]\right].
\end{align}
Equivalently, the goal is to make decisions $\{I_t^j: 1 \leq t\leq n_j ,\forall j\in[J]\}$ to minimize the regret in \eqref{eq_regret}. 

\begin{figure}[h]
\centering
\includegraphics[width = 1.1\columnwidth]{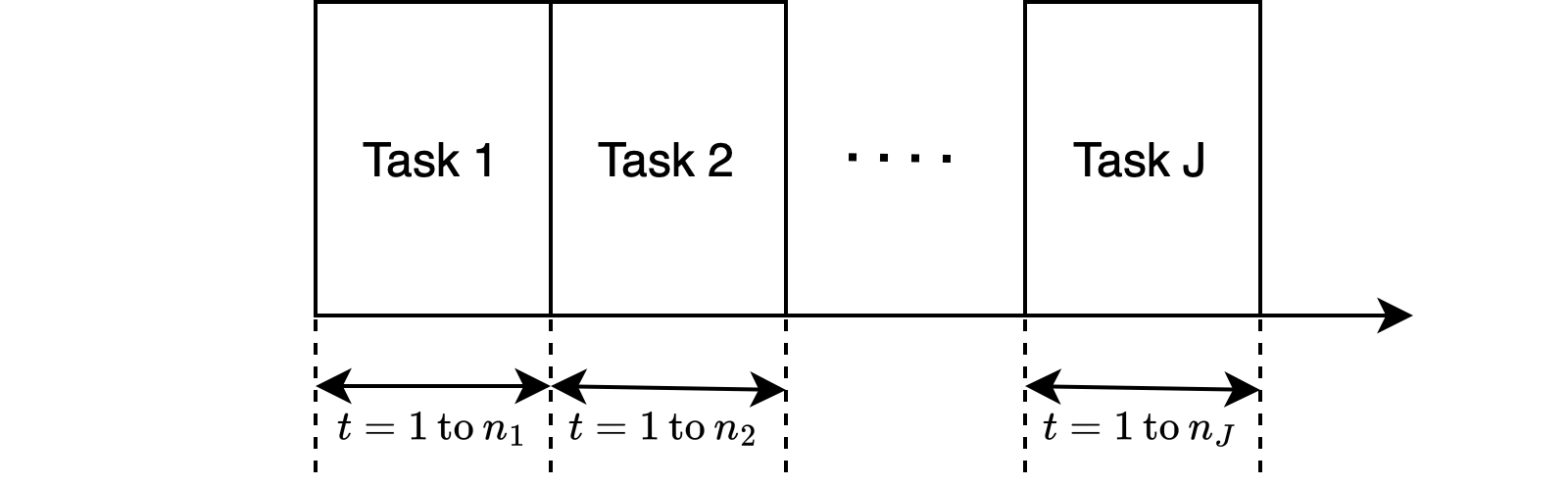}
\caption{Sequential multi-task bandit setting}
\label{fig:SeqMultiTaskSett}
\end{figure}
In this paper, we leverage the relation between the mean rewards of any two consecutive tasks (c.f. Assumption \ref{assump1}) in order to minimize the regret $R_J$. This is accomplished by reusing/transferring reward samples from previous tasks to make decisions in the current task. In the next section, we present two algorithms - one assumes that the parameter $\epsilon_k$ in Assumption \ref{assump1} is known, and the other for the case when $\epsilon_k$ is unknown. We provide the regret analysis of these algorithms and compare them with the baseline algorithm, which works without reusing/transferring reward samples from previous tasks.

\section{Algorithms and Regret Analysis}
\label{algo_RegAna}
Our algorithms are based on Upper Confidence Bound algorithm (UCB) \cite{auer2002finite} for bandits. UCB is a popular strategy for balancing exploration and exploitation by selecting arms based on the upper confidence bounds of their estimated mean rewards. A simple extension of UCB to the sequential bandit setting is to use UCB separately for each task. In particular, this approach uses reward samples only from the current task to compute the upper confidence bounds of the estimated mean rewards. We call this approach as No Transfer-UCB (NT-UCB) algorithm. However, NT-UCB does not leverage the relation between any two consecutive tasks (c.f. Assumption \ref{assump1}). In contrast, we leverage Assumption \ref{assump1} by transferring reward samples from the previous task to the current task. We call this algorithm Transfer-UCB (Tr-UCB). Next, we discuss these algorithms in detail.

\subsection{No Transfer-UCB (NT-UCB)}

The NT-UCB algorithm is a straightforward extension of the UCB algorithm \cite{auer2002finite} to the sequential multi-armed bandit task setting. It involves resetting the mean reward estimates and, therefore, the upper confidence bounds at the beginning of a new task. Essentially, when transitioning to a new task, NT-UCB discards reward samples from previous tasks and recomputes the confidence bounds based on the rewards obtained in the current task. In any task, the NT-UCB algorithm (shown in Algorithm \ref{alg:NT-UCB}) begins by pulling each arm once. From $t\geq K+1$, the algorithm computes the sample average estimates of mean reward denoted by,   
{\begin{align}
\label{eq:estimate1}
\hat{\mu}_{1k}^{j}(t) = \:\:\frac{\sum\limits_{\tau = 1}^{t}r_{I_\tau}^{j}\mathds{1}\{I_\tau = k\}}{N_k^j(t)},
\end{align}}
\noindent where $N_k^j(t)$ denotes the number of times arm $k$ is pulled until time $t$ in the task $j$.
Let $q_{1k}^j(t)$ denote its corresponding confidence width which is computed as follows,
{\begin{align}
\label{eq:confWidth1}
q_{1k}^j(t) = \sqrt{\frac{\alpha\log{t}}{2N_k^{j}(t)}},
\end{align}}
where $\alpha>2$. Next, the NT-UCB algorithm uses $\hat{\mu}_{1k}^{j}(t-1)+q_{1k}^j(t-1)$ to make decision $I_t^j$ at time $t$ in task $j$ based on the ``optimism in the face of uncertainty principle",
{\begin{align}
I_t^j = \underset{k\in [K]}{\arg \max}\left\{\hat{\mu}_{1k}^{j}(t-1)+q_{1k}^{j}(t-1)\right\}.
\end{align}}
 
Next, we provide the regret upper bound of the NT-UCB algorithm. Let $\Delta_k^{\max} = \max\limits_{j\geq 1} \: \{\Delta_k^j\} $ and $\Delta_k^{\min} = \min\limits_{j\geq 1,\Delta_{k}^{j}> 0} \: \{\Delta_k^j\}$ denote the universal (over all tasks) upper and lower bound on the sub-optimality gap of arm $k$. Let $\alpha>2$ be a positive integer. Then we have the following bound on the regret. 
\begin{lemma}
\label{lemma:ucbRegret}
The total pseudo-regret of NT-UCB satisfies 
{\begin{align}
\label{eq:ucbRegret}
R_J &\leq \sum\limits_{k=1}^{K}\left[\sum\limits_{{\substack{j=1 \\ \Delta_k^j > 0}}}^{J}\frac{2\alpha \log{n_j}}{\Delta_k^j}+\frac{\alpha}{\alpha-2}\sum\limits_{j=1}^{J} \Delta_k^j\right].
\end{align}}
\end{lemma}
\begin{proof}
Follows from the regret upper bound of standard UCB algorithm \cite{bubeck2012regret}.
\end{proof}
\begin{remark}
When the total number of tasks satisfies $J = \mathcal{O}(T^{\beta})$, where $\beta\in[0,1)$, the regret in Lemma \ref{lemma:ucbRegret} becomes $R_T = \mathcal{O}\bigg(\left(\sum\limits_{k=1}^{K}\frac{1}{\Delta_k^{\min}}\right)T^\beta\log(T)\bigg) +  \mathcal{O}\bigg(\left(\sum\limits_{k=1}^{K}\Delta_k^{\max}\right)T^\beta\bigg)=  \mathcal{O}(T^\beta\log(T))$. 
\end{remark}

\begin{algorithm}
\caption{NT-UCB}\label{alg:NT-UCB}
\begin{algorithmic}[1]
\Require Total tasks $J$, parameter $\alpha$, and number of arms $K$
\For {task $j = 1,2,...,J$}
    \For {$t = 1,\cdots,K$}
        \State $I_t^j = t$ (Pull each arm once)
    \EndFor
    \For {$t = K+1,\cdots,n_j$}
        \State compute $\hat{\mu}_{1k}^j(t-1)$ using \eqref{eq:estimate1}, $\forall k\in[K]$
        \State compute $q_{1k}^j(t-1)$ using \eqref{eq:confWidth1}, $\forall k\in[K]$
        \State select arm $I_t^j = \underset{k\in [K]}{\arg \max} \{\hat{\mu}_{1k}^j(t-1)+q_{1k}^j(t-1)\}$
        \State update  number of pulls $N_k^j(t)$, $\forall k\in[K]$
    \EndFor
\EndFor
\end{algorithmic}
\end{algorithm}

\subsection{Transfer-UCB (Tr-UCB) with parameter $\epsilon_k$ known}
The NT-UCB algorithm uses reward samples from the current task to make decisions $I_t$. However, by Assumption \ref{assump1}, the mean rewards of consecutive tasks are similar. Hence, samples from previous tasks contain information about the mean reward of the current task. To leverage this information, we construct an auxiliary estimate using the reward samples from the preceding task. Subsequently, the UCB and auxiliary estimates are combined to make decisions $I_t$. Next, we describe this Tr-UCB algorithm in detail (shown in Algorithm \ref{alg:Tr_UCB}).

Let $\hat{\mu}_{2k}^{j}(t)$ denote the auxiliary estimate of the mean reward of arm $k$ at time $t$ in task $j$. The auxiliary estimate $\hat{\mu}_{2k}^{j}(t)$ is computed using the reward samples from the current task $j$ and the preceding task $j-1$. Further, let $B_k = \frac{\eta-4\epsilon_k^2}{4\epsilon_k^2}$ denote the maximum number of transferred samples for arm $k$ from the preceding task, where $\eta>8$. The term $B_k$ is large when the parameter $\epsilon_k$ is small, which means a large number of reward samples from the preceding task are allowed to be transferred; on the other hand, fewer reward samples are transferred when $\epsilon_k$ is large. By limiting the number of transferred reward samples through $B_k$, the amount of bias introduced in the estimation of mean in the current task is not excessive and remains sufficient to facilitate the transfer. Then, the auxiliary estimate $\hat{\mu}_{2k}^{j}(t)$ is computed as:
{\begin{align}
\label{eq:auxEst}
\hat{\mu}_{2k}^{j}(t) = \:\frac{R_k^j+\sum\limits_{\tau = 1}^{t}r_{I_\tau}^{j}\mathds{1}\{I_\tau = k\}}{N_k^{j}(t)+M_k^j},
\end{align}}
where $R_k^{j} = \sum\limits_{\tau = 1}^{n_{j-1}}r_{I_\tau}^{j}\mathds{1}\{I_\tau = k; N_k^{j-1}(\tau)\leq B_k\}$ denotes the sum of transferred rewards, and $M_k^{j} = \min\{N_k^{j-1}(n_{j-1}), B_k\}$ denotes the number of transferred reward samples for the arm $k$. Note that for $\epsilon_k = 0$, we have $M_k^{j} = N_k^{j-1}(n_{j-1})$, which means all the reward samples from the preceding task are transferred. Further, observe that the reward samples from the previous tasks other than the preceding task also carry information about the mean reward of the current task and, therefore, can be used to compute the auxiliary estimate. This is particularly useful if the length of the preceding task is small, which means there are fewer reward samples to transfer, even though the term $B_k$ is very large. However, to keep the algorithm and the analysis simple, we have considered transferring reward samples only from the preceding task\footnote{The extension of the proposed algorithm to the case of transferring reward samples from multiple previous tasks is similar, and we defer it to future work.}. Next, we compute the confidence width denoted by $q_{2k}^j(t)$ of the auxiliary estimate $\hat{\mu}_{2k}^{j}(t)$ as follows:
{\begin{align}
\label{eq:confWidth_aux}
q_{2k}^j(t) = \sqrt{\frac{\eta\log{(B_k+t)}}{2\left(N_{k}^j(t)+M_k^j\right)}}.
\end{align}}
Finally, the Tr-UCB algorithm makes decision $I_t^j$ at time $t$ by combining the upper confidence bounds of the mean reward computed using the sample average estimate $\hat{\mu}_{1k}^{j}(t)$ and the auxiliary estimate $\hat{\mu}_{2k}^{j}(t)$ as follows:
\begin{align}
\label{eq:trUCB_deci}
I_t^j =& \underset{k\in [K]}{\arg \max}\{\min \{\hat{\mu}_{1k}^j(t-1)+q_{1k}^j(t-1),\hat{\mu}_{2k}^{j}(t-1)+\nonumber\\
&\hspace{2cm}q_{2k}^j(t-1)\}\}.
\end{align}
Taking the minimum of the upper confidence bounds of the two estimates gives a conservative upper bound on the true value of the mean reward, which allows the algorithm to be less optimistic. This leads to increased exploitation and decreased exploration, therefore leading to a reduction in regret.

Next, we provide the regret result of Tr-UCB. For simplicity, we introduce the following notations:
\begin{align}
\label{eq_u1_u2}
u_{1k}^{j} \triangleq \frac{2\alpha\log{(n_j)}}{(\Delta_k^{j})^2} \:\:\text{and}\:\: 
u_{2k}^{j} \triangleq \frac{2\eta\log{(B_k+n_j)}}{(\Delta_k^{j})^2}.
\end{align}
\begin{theorem}
\label{theorem1_trucb}
Let $u_{1k}^{j}$ and $u_{2k}^{j}$ be defined as in \eqref{eq_u1_u2}. The pseudo-regret of Tr-UCB satisfies 
{\begin{align}
\label{eq:trucbRegret}
R_J &\leq \sum\limits_{k=1}^{K}\Delta_k^{\max}\Bigg( \bigg(\sum\limits_{l=0}^{\lceil\frac{J-2}{2}\rceil}\min\left\{U_{k}^l,V_{k}^l\right\}\bigg)+W_k^J+ \nonumber\\
&\hspace{2.5cm}J\bigg(\frac{\alpha}{\alpha-2}+\frac{8}{\eta-8}\bigg)\Bigg),
\end{align}}
where,
{\begin{align*}
U_{k}^l = u_{1k}^{2l+1}\mathds{1}\{\Delta_{k}^{2l+1}> 0\}+u_{1k}^{2l+2}\mathds{1}\{\Delta_{k}^{2l+2}> 0\},
\end{align*}}
\vspace{-0.5cm}
{\begin{align*}
V_{k}^l&=\Bigg(u_{2k}^{2l+1}\mathds{1}\{\Delta_{k}^{2l+1}> 0\}+u_{2k}^{2l+2}\mathds{1}\{\Delta_{k}^{2l+2}> 0\}\Bigg)-\\
&\hspace{1cm}\min\Bigg\{\max\bigg\{u_{2k}^{2l+1},u_{2k}^{2l+2}\bigg\},B_k\Bigg\}
\end{align*}}
\vspace{-0.5cm}
{\begin{align*}
\text{and}\:\:W_{k}^J= \mathds{1}\{J \: \text{is odd},\Delta_k^J>0\}\Bigg(\min\bigg\{u_{1k}^J, u_{2k}^J\bigg\}\Bigg).
\end{align*}}
\end{theorem}
\begin{proof}
Refer to the Appendix.
\end{proof}
\begin{remark}
Observe that when the total number of tasks $J = \mathcal{O}(T^{\beta})$, where $\beta\in[0,1)$, the regret in \eqref{eq:trucbRegret} follows $R_T = \mathcal{O}(T^\beta\log(T))$, which is the same as for NT-UCB. Therefore, Tr-UCB ensures there is no negative transfer irrespective of the value of $\epsilon_k$. In other words, Tr-UCB guarantees performance is not negatively affected when the tasks are dissimilar, i.e. when the value of $\epsilon_k$ is high. 
\end{remark}
Although the algorithms have the same order of growth with respect to $T$, we show the benefit of transfer by comparing the actual regret expressions.

\textbf{Benefit of Transfer}. We show the benefit of transfer by comparing the expressions of regret in \eqref{eq:trucbRegret} and \eqref{eq:ucbRegret}. The first term in the regret bound captures the benefit of transfer, and therefore we compare the first terms of \eqref{eq:trucbRegret} and \eqref{eq:ucbRegret}, respectively. Define the following terms for the tasks $2l+1$ and $2l+2$, for some $l\in \{0,1,\cdots,\lceil \frac{J-2}{2}\rceil\}$
\begin{align*}
A_k^l \triangleq \Delta_k^{\max}U_k^l,\hspace{1cm}
E_k^l \triangleq \Delta_k^{\max}V_k^l,
\end{align*}
\begin{align*}
F_k^l \triangleq u_{1k}^{2l+1}\Delta_{k}^{2l+1}+u_{1k}^{2l+2}\Delta_{k}^{2l+2}
\end{align*}
Further, we rewrite the regret upper bound of NT-UCB in \eqref{eq:ucbRegret} using $F_k^l $ as follows,
{\begin{align}
R_J &\leq \sum\limits_{k=1}^{K}\bigg(\bigg(\sum\limits_{{\substack{l=0 \\ \Delta_k^{2l+1} > 0, \Delta_k^{2l+2} > 0}}}^{\lceil\frac{J-2}{2}\rceil}F_{k}^l\bigg)+\frac{\alpha}{\alpha-2}\sum\limits_{j=1}^{J} \Delta_k^j\bigg),
\end{align}}

We analyze the benefit of transfer for the consecutive tasks $2l+1$ and $2l+2$. Note that for the transfer to be useful for the tasks $2l+1$ and $2l+2$, we need $\min\{A_k^l,E_k^l\}<F_k^l$. Since $A_k^l\geq F_k^l$, this can happen only if $E_k^l<F_k^l$. Observe that when $B_k$ is very small, the term $E_k^l$ is relatively large, and as $B_k$ increases, the term $E_k^l$ decreases in comparison with $F_k^l$. Hence, for some large enough $B_k$, we get $E_k^l<F_k^l$, which leads to a decrease in the regret upper bound of Tr-UCB. Recall that $B_k$ depends on the parameter $\epsilon_k$ through an inverse relationship, i.e. as $\epsilon_k$ decreases $B_k$ increases. Therefore, the regret upper bound of Tr-UCB decreases when compared to NT-UCB, when the parameter $\epsilon_k$ decreases.

\begin{algorithm}
\caption{Tr-UCB - Parameter $\epsilon_k$ is known}\label{alg:Tr_UCB}
\begin{algorithmic}[1]
\Require Total tasks $J$, number of arms $K$, and parameters $\alpha$,$\eta$, $\epsilon_k,\forall k\in[K]$ 
\For {task $j = 1,2,...,J$}
    \State repeat steps 2 to 4 of Algorithm \ref{alg:NT-UCB}
    \For {$t = K+1,\cdots,n_j$}
        \State compute $\hat{\mu}_{1k}^j(t-1)$ using \eqref{eq:estimate1}, $\forall k\in[K]$
        \State compute $q_{1k}^j(t-1)$ using \eqref{eq:confWidth1}, $\forall k\in[K]$
        \State compute $\hat{\mu}_{2k}(t-1)$ using \eqref{eq:auxEst},$\forall k\in [K]$
        \State compute $q_{2k}^j(t-1)$ using \eqref{eq:confWidth_aux}, $\forall k\in [K]$
        \State select arm $I_t^j$ using \eqref{eq:trUCB_deci} at time $t$
        \State update  number of pulls $N_k^j(t)$, $\forall k \in [K]$
    \EndFor
\EndFor
\end{algorithmic}
\end{algorithm}

\subsection{Transfer-UCB (Tr-UCB2) with parameter $\epsilon_k$ unknown}
The Tr-UCB algorithm discussed previously assumes that the parameter values $\epsilon_k$ are known. However, access to these parameters may not be available in practice. One approach to address this problem is to estimate the value of $\epsilon_k$ using the past reward samples and then follow the Tr-UCB algorithm using the estimated value. We employ the same approach but with a minor modification to the Tr-UCB algorithm. To increase the confidence in the estimates of $\epsilon_k$, we pull the arms uniformly for a fixed number of steps. This increases the number of reward samples for each arm, and therefore, increases the confidence in the estimates of $\epsilon_k$. We call this approach Tr-UCB2. Next, we describe Tr-UCB2 algorithm in detail (refer to Algorithm \ref{alg:Tr_UCB2} for pseudo-code).
\begin{figure}[h]
\centering
\includegraphics[width = 1 \columnwidth]{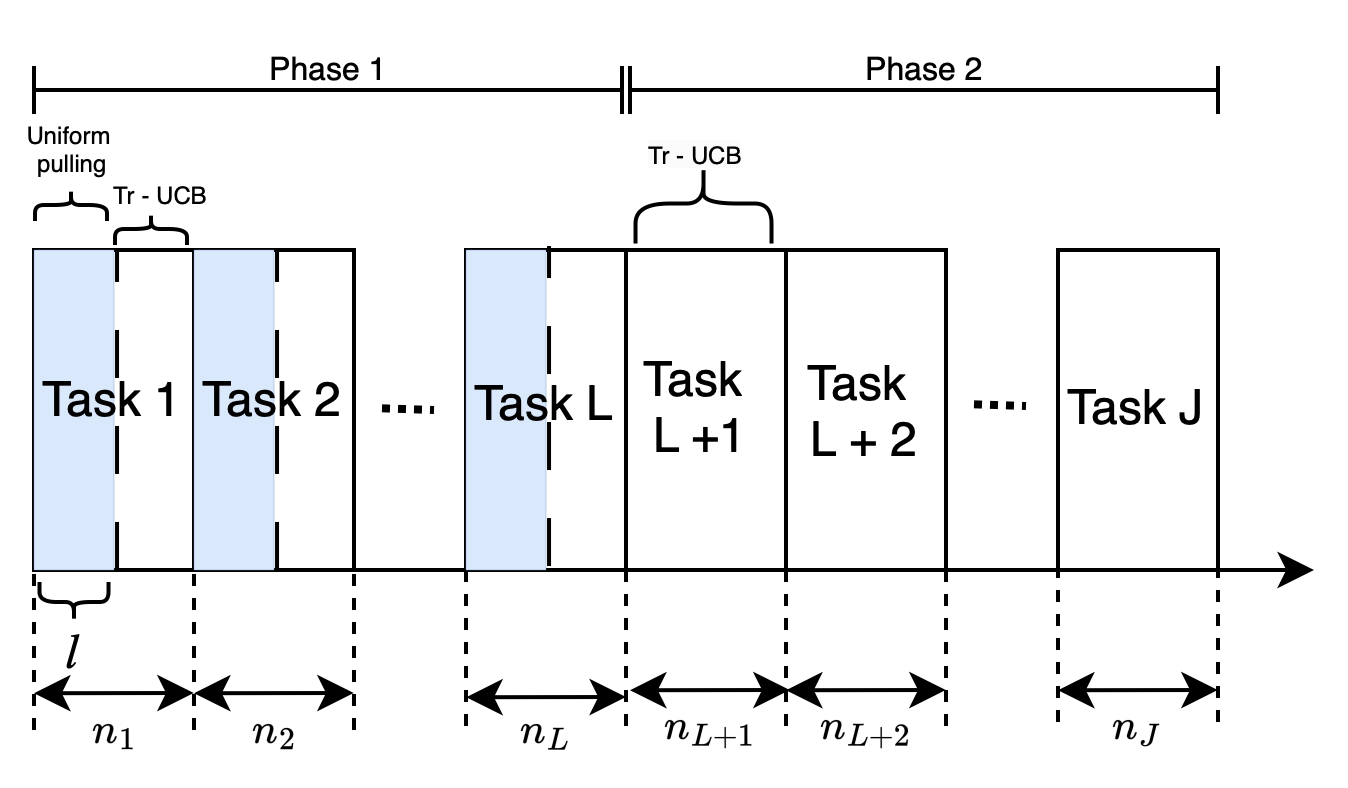}
\caption{Pictorial representation of Tr-UCB2}
\label{fig:TR-UCB2}
\end{figure}

As depicted in Figure \ref{fig:TR-UCB2}, the algorithm's behavior is divided into two phases. Let $L\geq 2$ denote the number of tasks in phase I. For each task in phase I, the Tr-UCB2 algorithm begins by pulling the arms uniformly for steps $1\leq t\leq l $, where $l = aK$, for some $a\in \mathbb{N}$. This ensures enough reward samples from each arm are generated for estimating $\epsilon_k$ with good confidence. For $t>l$, the algorithm uses an estimate of $\epsilon_k$ (described below) and follows the Tr-UCB strategy. The estimate of $\epsilon_k$ is computed at the beginning of every task using the reward samples from the previous tasks. After $L$ tasks, the algorithm enters Phase II which is similar to Phase I except that there is no uniform sampling in the tasks of Phase II. Similar to Phase I, an estimate of $\epsilon_k$ is computed at the beginning of every task in Phase II and Tr-UCB strategy is used. Algorithm \ref{alg:Tr_UCB2} mentions all steps of TR-UCB2.

%Next in phase II, similarly to phase I, an estimate of $\epsilon_k$ is computed at the beginning of every task, and using the estimate the agorithm follows Tr-UCB strategy without employing uniform pulling of arms. 

Next, we explain the computation of the estimate of $\epsilon_k$ in detail. Let $\hat{\epsilon}_k^j$ denote an estimate of the parameter $\epsilon_k$ for arm $k$ in task $j$. Then, $\hat{\epsilon}_k^j$ is computed as follows:  
\begin{align}
\label{eq:epsilon_hat}
 \hat{\epsilon}_k^j \hspace{-0.1cm}= \hspace{-0.1cm}\begin{cases}
1 & \hspace{-0.5cm}\mbox{ \text{,}\:$1\leq j\leq 2 $}\\
\hspace{-0.1cm}\max\limits_{i\in[j-1], c_k^i\leq c_0}\hspace{-0.2cm}\{ |\hat{\mu}_{1k}^{i+1}(n_{i+1})-\hat{\mu}_{1k}^{i}(n_{i})\pm c_k^i|\} & \hspace{-0.5cm}\mbox{ \text{,}\:$ 2 < j\leq J $}\\
\end{cases}
\end{align}
\vspace{-0.1cm}
where 
\begin{align*}
c_k^i = \sqrt{\frac{N_k^{i+1}(n_{i+1})+N_k^{i}(n_{i})}{2N_k^{i+1}(n_{i+1})N_k^{i}(n_{i})}\log{\left(\frac{2}{\delta}\right)}},\:\:\text{and}\:\:\delta \in (0,1) ,
\end{align*}
\vspace{-0.1cm}
\begin{align*}
c_0 = \sqrt{\frac{K}{l}\log\left(\frac{2}{\delta}\right)}.
\end{align*}
We outline the motivation for the expression of the estimate $\hat{\epsilon}_k^j$ in the following three steps:

(i) From Assumption \ref{assump1}, we know that $|\mu^{i}_k - \mu^{i+1}_k|\leq \epsilon_k$. Hence, we compute a high probability upper bound on the true value of $\epsilon_k$ using the difference between the empirical averages of the mean rewards, i.e., $\hat{\mu}_{1k}^{i+1}(n_{i+1})-\hat{\mu}_{1k}^{i}(n_{i})$. By using Hoeffding's inequality \cite{hoeffding1994probability}, it can be shown that with probability at least $1-\delta$, the true difference of the mean rewards $\mu^{i}_k - \mu^{i+1}_k$ lies in the interval $[\hat{\mu}_{1k}^{i+1}(n_{i+1})-\hat{\mu}_{1k}^{i}(n_{i})-c_k^i, \hat{\mu}_{1k}^{i+1}(n_{i+1})-\hat{\mu}_{1k}^{i}(n_{i})+c_k^i]$. 

(ii) By considering $|\hat{\mu}_{1k}^{i+1}(n_{i+1})-\hat{\mu}_{1k}^{i}(n_{i})\pm c_k^i|$, we are essentially taking the maximum possible value of the difference in the mean rewards from the confidence interval $[\hat{\mu}_{1k}^{i+1}(n_{i+1})-\hat{\mu}_{1k}^{i}(n_{i})-c_k^i, \hat{\mu}_{1k}^{i+1}(n_{i+1})-\hat{\mu}_{1k}^{i}(n_{i})+c_k^i]$. This method is pessimistic (biased towards higher values of $\hat{\epsilon}_k$) in estimating the true value of $\epsilon_k$, thus helping in minimizing the negative transfer, especially when the estimates are not sufficiently accurate.

(iii) Since the parameter $\epsilon_k$ upper bounds the difference in the mean rewards, we take the maximum value among all the estimates of the difference in the mean rewards. Note that the constraint $c_k^i\leq c_0$ ensures the estimates of the difference in mean rewards $\hat{\mu}_{1k}^{i+1}(n_{i+1})-\hat{\mu}_{1k}^{i}(n_{i})$ have good confidence. This is because, the constraint $c^i_k \leq c_0$ ensures that $c^i_k$ is never too large, in other words, the estimates computed using fewer samples are not considered. Therefore the constraint ensures the estimates have high confidence. 

In the following theorem, we provide the regret analysis of Tr-UCB2,

\begin{theorem}
\label{theorem2_trucb}
The pseudo-regret of Tr-UCB2 satisfies 
{\begin{align}
\label{eq:trucbRegret2}
R_J &\leq \sum\limits_{k=1}^{K}\Delta_k^{\max}\bigg(\frac{lL}{K}+\sum\limits_{j=1}^{J}u_{1k}^j+ J\bigg(\frac{\alpha}{\alpha-2}+\frac{8}{\eta-8}\bigg)\nonumber\\
&\hspace{2.5cm}+TJ\delta\bigg).
\end{align}}
Furthermore, if the total number of tasks satisfies $J = \mathcal{O}(T^{\beta})$ and confidence parameter satisfies $\delta = \frac{1}{T}$, then the regret follows $R_T = \mathcal{O}(T^\beta\log(T))$.
\end{theorem}
\begin{proof}
Refer to the Appendix.
\end{proof}
Next, we analyze the regret upper bound \ref{eq:trucbRegret2} of Tr-UCB2 algorithm. The first term in \eqref{eq:trucbRegret2} captures the increased regret due to the uniform pulling of arms in phase I. This means that increasing the parameters $l$ and $L$, increases the regret. However, to estimate $\epsilon_k$ with high confidence, the parameters $l$ and $L$ need to be chosen sufficiently large. Hence, there is a tradeoff in estimating $\epsilon_k$ with high confidence and reducing the regret contributed by uniformly pulling the arms. The second term of \eqref{eq:trucbRegret2} does not explicitly capture the transfer benefit, unlike the regret upper bound of Tr-UCB in \eqref{eq:trucbRegret}. However, Tr-UCB2 is atleast as good as NT-UCB since both follow the same regret order and the empirical results in section \ref{simulation} further demonstrate that Tr-UCB2 performs better than NT-UCB. Nevertheless, this result is somewhat loose and does not fully capture the transfer benefit, which we aim to improve in future work. The last term captures the increased regret due to the error in the estimate of $\epsilon_k$. By decreasing $\delta$, the regret due to the last term decreases. However, the parameter $l$ needs to be increased simultaneously in order to increase confidence in the estimate of $\epsilon_k$, showing a trade-off.  

\begin{algorithm}
\caption{Tr-UCB2 - Parameter $\epsilon_k$ is unknown}\label{alg:Tr_UCB2}
\begin{algorithmic}[1]
\Require Parameters $\alpha$, $l$, $L$, $K$ and $J$
\For {task $j = 1,2,...,J$}
    \State compute estimate $\hat{\epsilon}_k^j,\forall k\in[K]$
    \If{$j\leq L$} 
        \For {$t = 1,\cdots,l$}
            \State $I_t = t\bmod K$ 
        \EndFor 
        \For {$t = l+1,\cdots,n_j$}
        \State repeat steps 4 to 9 of Algorithm \ref{alg:Tr_UCB} using $\hat{\epsilon}_k^j$
        \EndFor
    \Else
        \State repeat steps 2 to 10 of Algorithm \ref{alg:Tr_UCB} using $\hat{\epsilon}_k^j$
    \EndIf 
\EndFor
\end{algorithmic}
\end{algorithm}

\section{Numerical Simulations}
\label{simulation}
In this section, we present the simulation results showing the empirical improvement of algorithms Tr-UCB, Tr-UCB2 over NT-UCB and a naive transfer algorithm (called Naive-Transfer), indicating the benefit of transfer. Naive-Transfer is an empirical algorithm that transfers all reward samples from the preceding task, assuming that the samples come from the same distribution. We consider a sequence of tasks, where each task is a multi-armed bandit with $K = 5$ arms. The total number of tasks $J = 1000$, with task length $n_j = 10000,\forall j\in [J]$. The mean rewards $\{\mu_k^1\}_{k=1}^K$ of the first task are generated by uniformly sampling from the $[0,1]$ interval. The mean rewards of the subsequent tasks have to satisfy Assumption \ref{assump1}. Towards this end, we construct uniform distributions of mean $\mu_k^1$ and width $2\epsilon_k$ for each arm $k$. If the support of any distribution lies outside $[0, 1]$ interval, then we appropriately adjust the width. Then, the mean rewards $\mu_k^2$ of task 2 are uniformly sampled from these distributions. The mean rewards of the subsequent tasks are generated from the mean rewards of the preceding task in a similar manner. The rewards for arm $k$ in any task $j$ are generated by sampling from uniform distributions of mean $\mu_k^j$ with width $d = 0.1$. Once again, the widths are adjusted if they fall outside the $[0, 1]$ interval.
\begin{figure}[thpb]
\begin{subfigure}{\columnwidth}
    \centering
    \includegraphics[width=\columnwidth]{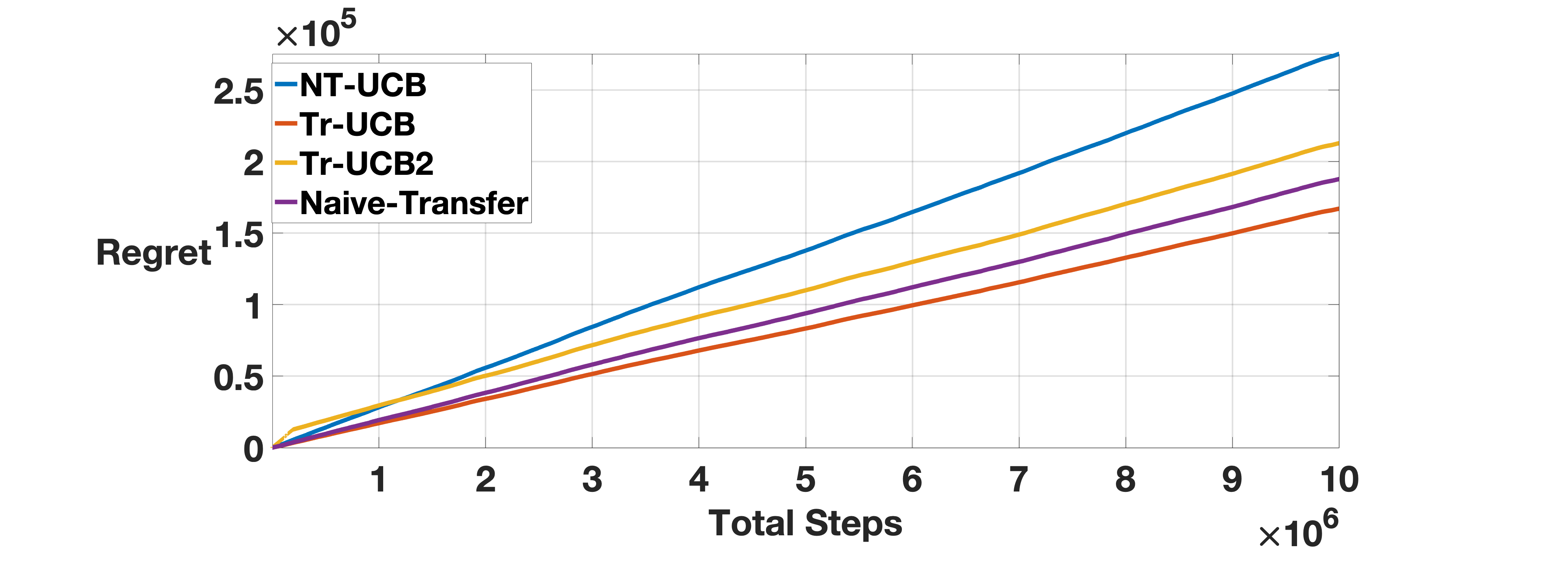}
    \caption{$\epsilon_k = 0.05$}
    \label{fig:1}
    \end{subfigure}
\begin{subfigure}{\columnwidth}
    \centering
    \includegraphics[width=\columnwidth]{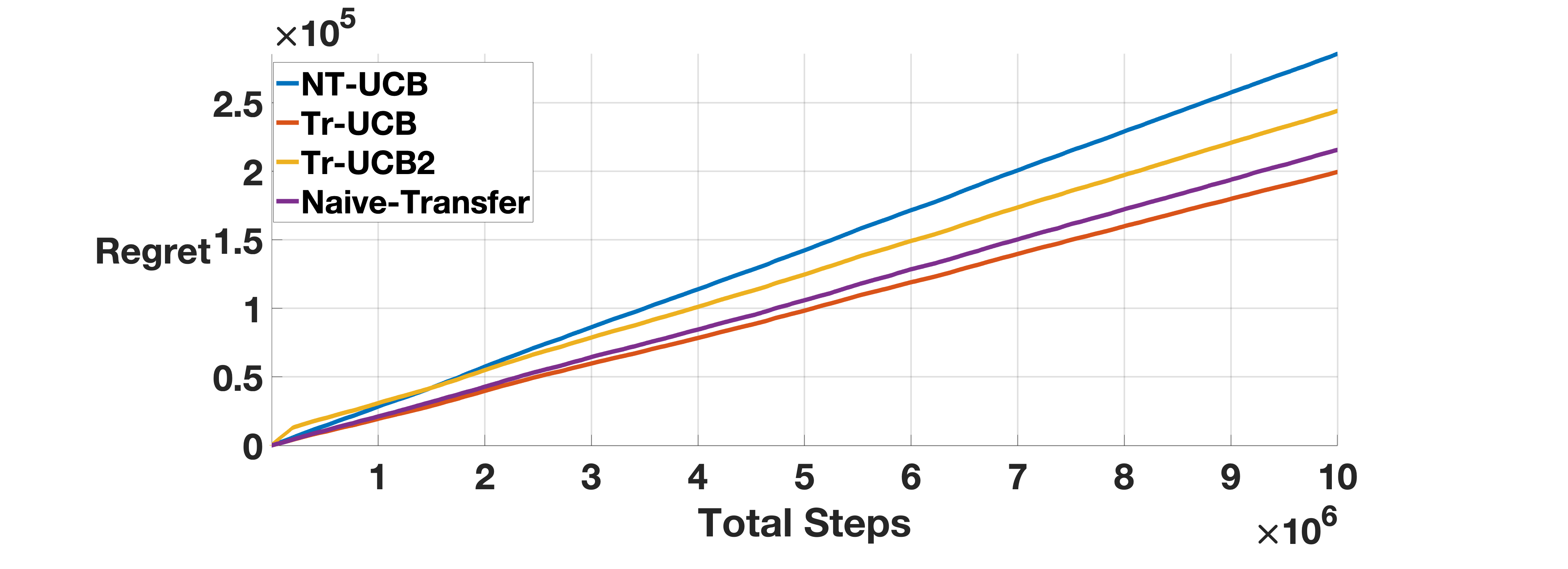}
    \caption{$\epsilon_k = 0.1$}
    \label{fig:2}
\end{subfigure}
\begin{subfigure}{\columnwidth}
    \centering
    \includegraphics[width=\columnwidth]{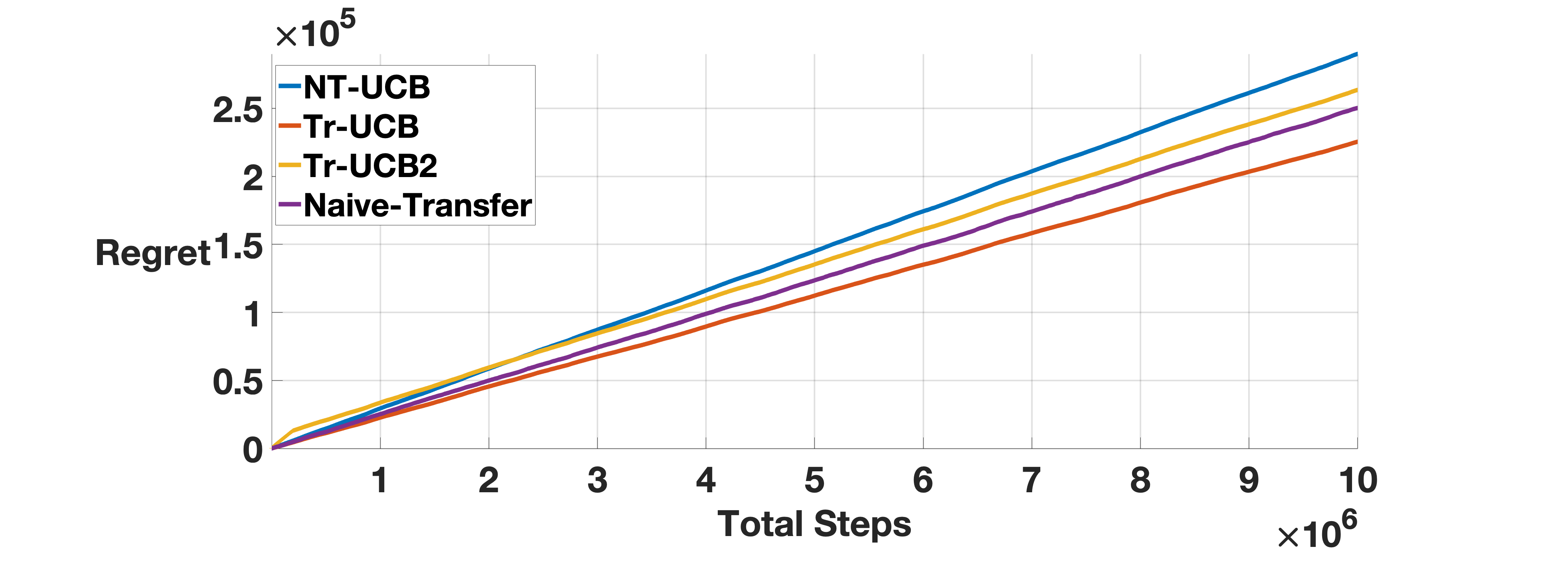}
    \caption{$\epsilon_k = 0.15$}
    \label{fig:3}
\end{subfigure}
\caption{Empirical Regret Vs Total Steps of  NT-UCB, Tr-UCB, Tr-UCB2 and Naive-Transfer algorithms for different values of $\epsilon_k$}
        \label{fig:all_1}
\end{figure}

Figure \ref{fig:all_1} and \ref{fig:all_2} shows the empirical regret over the total steps of algorithms NT-UCB, Tr-UCB, Tr-UCB2, and Naive-Transfer for different values of $\epsilon_k$. We have considered $\epsilon_k=\epsilon, \forall k\in[K]$. The plots are generated by averaging over $20$ realizations. The parameter values used in Tr-UCB2 algorithm are $l =2000$, $L = 20$, $\delta = 0.1$ and $\alpha=\eta=8.1$. The same value of $\alpha$ is used for NT-UCB algorithm. 
\begin{figure}[thpb]
\begin{subfigure}{\columnwidth}
    \centering
    \includegraphics[width=\columnwidth]{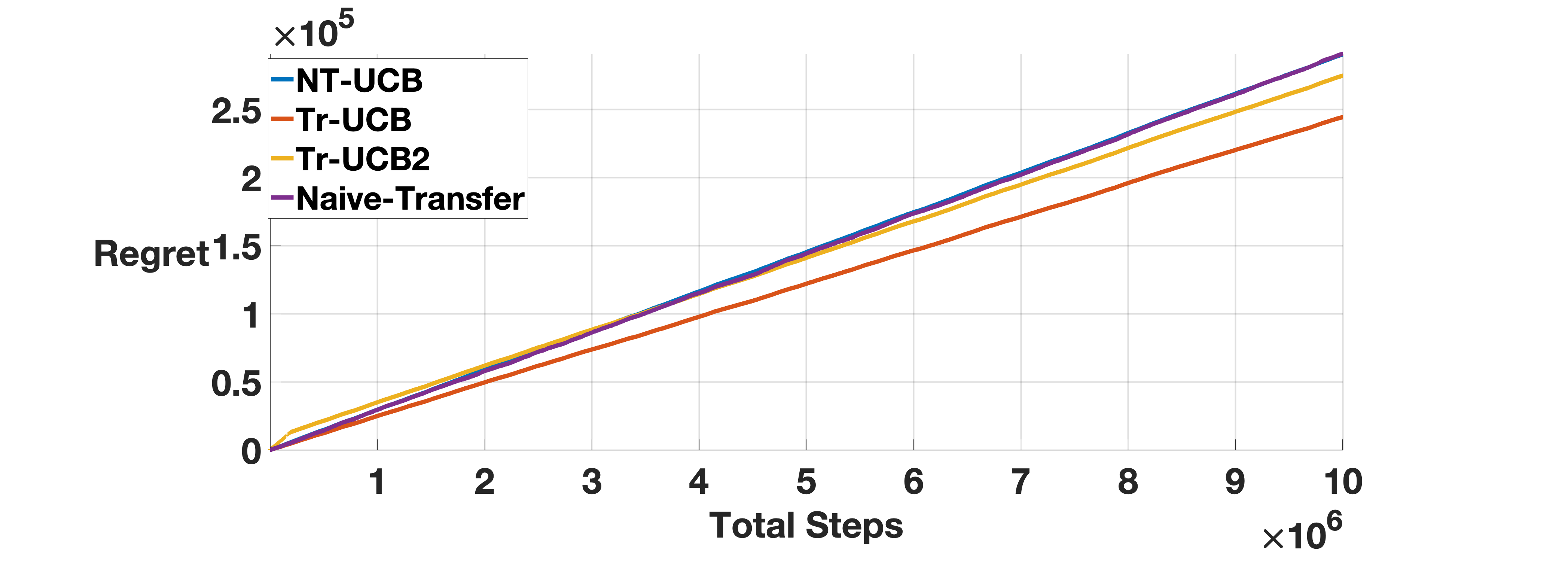}
    \caption{$\epsilon_k = 0.2$}
    \label{fig:4}
    \end{subfigure}
\begin{subfigure}{\columnwidth}
    \centering
    \includegraphics[width=\columnwidth]{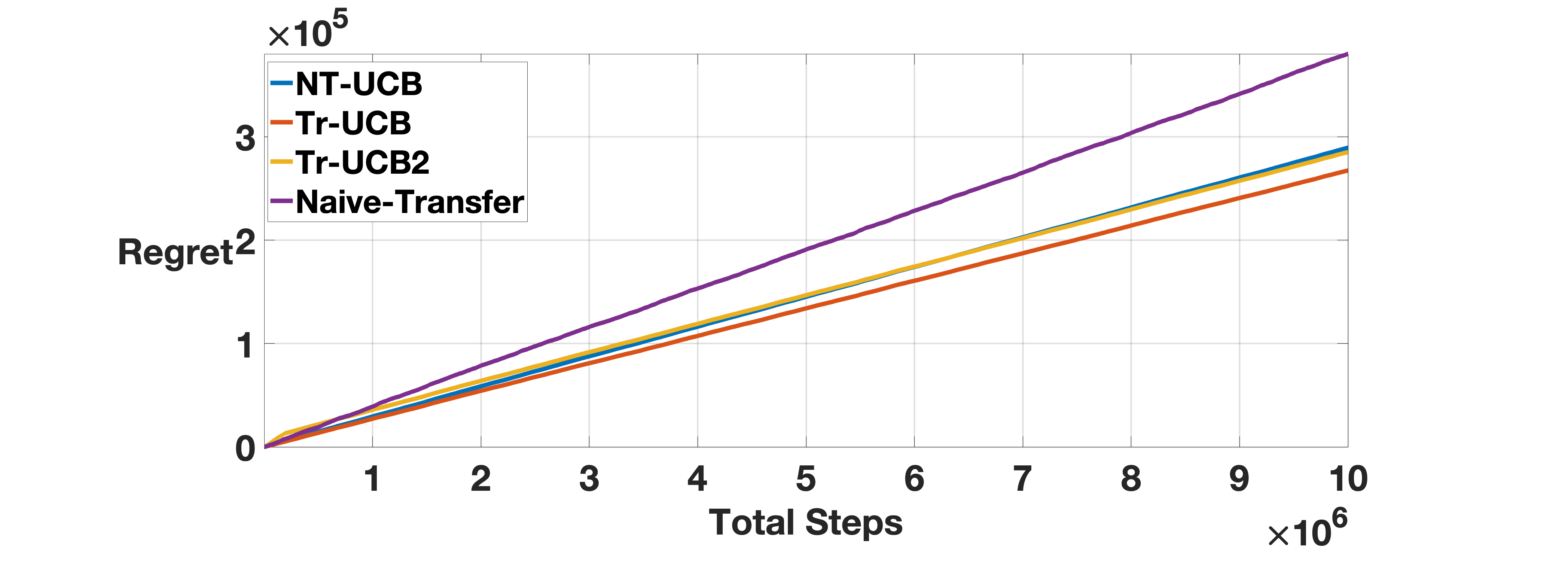}
    \caption{$\epsilon_k = 0.3$}
    \label{fig:5}
\end{subfigure}
\begin{subfigure}{\columnwidth}
    \centering
    \includegraphics[width=\columnwidth]{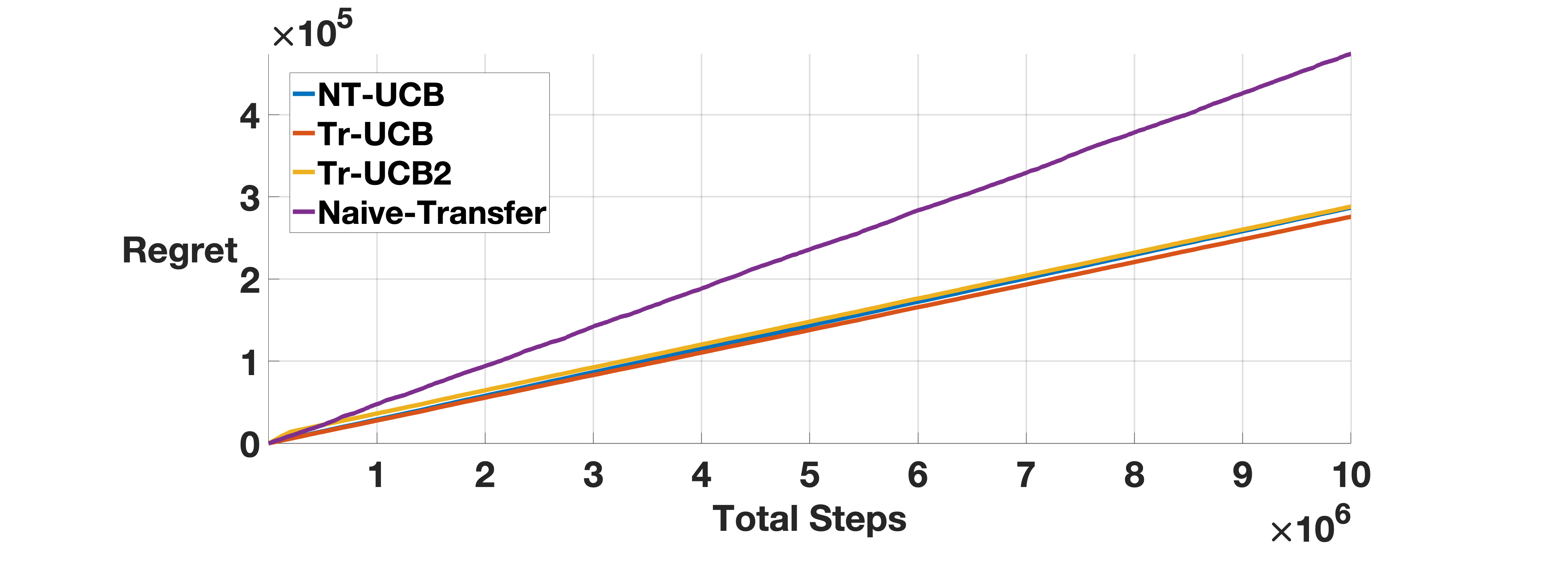}
    \caption{$\epsilon_k = 0.4$}
    \label{fig:6}
\end{subfigure}
\caption{Empirical Regret Vs Total Steps of  NT-UCB, Tr-UCB, Tr-UCB2 and Naive-Transfer algorithms for different values of $\epsilon_k$}
        \label{fig:all_2}
\end{figure}

Next, we analyze the results in Figure \ref{fig:all_1} and \ref{fig:all_2}. For each value of $\epsilon_k$, the regret of the Tr-UCB algorithm is the lowest among all algorithms, which means transferring the reward samples with the knowledge of $\epsilon_k$ helps reduce the regret significantly. Next, observe that Tr-UCB2 incurs more regret than Tr-UCB since $\epsilon_k$ is not known. Also, both Tr-UCB and Tr-UCB2 have significantly lower regret than NT-UCB, especially when $\epsilon_k$ values are small and lie below NT-UCB overall $\epsilon_k$ values considered (no negative transfer). This implies transferring reward samples from preceding tasks is beneficial. On the other hand, Naive-Transfer performs well when $\epsilon_k$ values are small, with increasing $\epsilon_k$ values, the performance starts degrading and performs worse than NT-UCB for larger values of $\epsilon_k$. The reason for this behavior is that transferring a large number of reward samples naively introduces a large bias in the mean estimates of the reward in the current task. Further, Tr-UCB2 incurs high regret in the initial tasks because of uniform sampling of the arms and inaccurate estimates of $\epsilon_k$. However, it performs better as more tasks are encountered.

\section{CONCLUSIONS}

We considered a sequential multi-task setting, where each task is a stochastic multi-armed bandit. We introduced the parameter $\epsilon_k$ to capture the adjacent similarity between tasks. We analyzed the transfer of reward samples and proposed two transfer algorithms based on UCB; one assumes the knowledge of $\epsilon_k$, and the other estimates this parameter from data. We provided a regret analysis of the algorithms and validated our approach via numerical experiments. One of the future research directions is to address the gap between the performance of Tr-UCB and Tr-UCB2. Other research directions include computation of lower bound on the regret, and studying transfer learning in the context of varying task similarity parameters, non-stationary bandits and reinforcement learning. 


%%%%%%%%%%%%%%%%%%%%%%%%%%%%%%%%%%%%%%%%%%%%%%%%%%%%%%%%%%%%%%%%%%%%%%%%%%%%%%%%

%%%%%%%%%%%%%%%%%%%%%%%%%%%%%%%%%%%%%%%%%%%%%%%%%%%%%%%%%%%%%%%%%%%%%%%%%%%%%%%%

%%%%%%%%%%%%%%%%%%%%%%%%%%%%%%%%%%%%%%%%%%%%%%%%%%%%%%%%%%%%%%%%%%%%%%%%%%%%%%%%
\section*{APPENDIX}
To keep the notation simple, we re-denote several variables as, $\mu = \mu_k^j$, $\mu_* = \mu_*^j$, $\hat{\mu}_{1} = \hat{\mu}_{1k}^{j}(t-1)$, $\hat{\mu}_{1*} = \hat{\mu}_{1k_{*}^{j}}^{j}(t-1)$, $\hat{\mu}_{2} = \hat{\mu}_{2k}(t-1)$, $\hat{\mu}_{2*} = \hat{\mu}_{2k_{*}^{j}}(t-1)$, $\hat{M}_k^j = \min \{N_k^{j-1}(n_{j-1}),\hat{B}_k\}$
\small\begin{align*}
&q_{1} = \sqrt{\frac{\alpha\log{(t-1)}}{2N_k^{j}(t-1)}}, \:\:q_{1*} = \sqrt{\frac{\alpha\log{(t-1)}}{2N_{k_*^j}^{j}(t-1)}},\\
&q_{2} = \sqrt{\frac{\eta\log{(B_k+t-1)}}{2(N_{k}^j(t)+M_k^j)}},\:\:q_{2*} = \sqrt{\frac{\eta\log{(B_k+t-1)}}{2(N_{k_*^j}^j(t)+M_k^j)}},\\
&\hat{q}_{2} = \sqrt{\frac{\eta\log{(\hat{B}_k+t-1)}}{2(N_{k}^j(t)+\hat{M}_k^j)}},\:\:
\hat{q}_{2*} = \sqrt{\frac{\eta\log{(\hat{B}_k+t-1)}}{2(N_{k_*^j}^j(t)+\hat{M}_k^j)}}.
\end{align*}

\subsection{Proof of Theorem \ref{theorem1_trucb}}
\label{proof:theorem_2}
For arm $k$ to be pulled at time $t$ ($I_t = k$), at least one of the following five conditions needs to be true:
\begin{align}
\hat{\mu}_{1} - q_1 &> \mu, \label{eq:cond1} \\
\hat{\mu}_{1*} + q_{1*} &\leq \mu_{*},
\label{eq:cond2}\\
\hat{\mu}_{2*} + q_{2*} &\leq \mu_{*}, \label{eq:cond3}\\
\hat{\mu}_{2} - q_2 & > \mu, \label{eq:cond4}
\end{align}
\vspace{-0.8cm}
\begin{align}
N_k^{j}(t-1)  < u_{1k}^{j}\hspace{0.1cm}\text{and}\hspace{0.1cm}\
N_k^{j}(t-1)+M_k^j<  u_{2k}^{j}. \label{eq:cond5_2}
\end{align}

We prove the above claim by contradiction. Assume that the first condition in \eqref{eq:cond5_2} is false and none of the conditions in \eqref{eq:cond1}-\eqref{eq:cond4} is true. Then, using $q_1< \sqrt{\frac{\alpha\log{n_j}}{2N_k^{j}(t-1)}} $, we get
{\begin{align} 
\hat{\mu}_{1*} + q_{1*} &> \mu_{*} =\Delta_k^j + \mu
\geq 2q_{1}+\mu \geq\hat{\mu}_{1} + q_{1}, \label{eq:falseCond1} \\
\hat{\mu}_{2*} + q_{2*}&> \mu_{*} 
=\Delta_k^j + \mu \geq 2q_{1}+\mu
\geq\hat{\mu}_{1} + q_{1}. \label{eq:falseCond2}
\end{align}}
Conditions in \eqref{eq:falseCond1} and \eqref{eq:falseCond2} imply
{\begin{align} 
\label{eq:falseMainCond1}
\min\{\hat{\mu}_{1*} + q_{1*},\hat{\mu}_{2*} + q_{2*}\}
>\hat{\mu}_{1} + q_{1}.
\end{align}}
Similarly, when the second condition in \eqref{eq:cond5_2} is false and none of the conditions in \eqref{eq:cond1}-\eqref{eq:cond4} is true and, we get
{\begin{align} 
\label{eq:falseMainCond2}
\min\{\hat{\mu}_{1*} + q_{1*},\hat{\mu}_{2*} + q_{2*}\}
>\hat{\mu}_{2} + q_{2}.
\end{align}}
Thus, at least one of \eqref{eq:falseMainCond1} and \eqref{eq:falseMainCond2} is true, and this gives
{\begin{align*}
\min\{\hat{\mu}_{1*} + q_{1*},&\hat{\mu}_{2*} + q_{2*}\}
>\min\{\hat{\mu}_{1} + q_{1},\hat{\mu}_{2} + q_{2}\}.
\end{align*}}
The above condition means that the Tr-UCB algorithm will not pull arm $k$, and hence, there is a contradiction. The total regret over total tasks $J$ is given by
{\begin{align}
\label{eq:regretExp_2}
R_J = \sum\limits_{j=1}^{J}\sum\limits_{k=1}^K \Delta_{k}^{j}\mathbb{E}[N_k^j (n_j)],\leq \sum\limits_{k=1}^{K}\Delta_k^{\max}\mathbb{E}[D_k (J)],
\end{align}}
where $D_k(J)$ is the total number of sub-optimal pulls to arm $k$ over all tasks.
Therefore, we bound the regret by bounding the term $\mathbb{E}[D_k(J)]$. Consider an arbitrary sequence of decisions $I_t^j$, $t=1,2,\cdots,n_j$, $\forall j\in[J]$, we get 
\vspace{-0.1cm}
{\begin{align}
\label{eq:S_k}
&D_k(J)=\sum\limits_{j=1}^{J}\sum\limits_{t=1}^{n_j}\mathds{1}\{ I_t^j = k,k\neq k_{*}^j\}\nonumber\\
&=\sum\limits_{j=1}^{J}\bigg(\mathds{1}\{k\neq k_{*}^j\}+\sum\limits_{t=K+1}^{n_j}\mathds{1}\{ I_t^j = k,k\neq k_{*}^j\}\bigg),\nonumber\\
&=\sum\limits_{j=1}^{J}\sum\limits_{t=K+1}^{n_j}\mathds{1}\{ I_t^j = k,k\neq k_{*}^j;\eqref{eq:cond5_2}\hspace{0.1cm} \text{is}\hspace{0.1cm} \text{True}\}+\nonumber\\
&\hspace{0.1cm}\bigg(\mathds{1}\{k\neq k_{*}^j\}+\sum\limits_{t=K+1}^{n_j}\mathds{1}\{ I_t^j = k,k\neq k_{*}^j;\eqref{eq:cond5_2}\hspace{0.1cm} \text{is}\hspace{0.1cm} \text{False}\}\bigg)
\end{align}}
\vspace{-0.3cm}
{\begin{align*}
\label{eq:S_k_first}
&\hspace{-0.8cm}\text{First term in \eqref{eq:S_k}} = \sum\limits_{j=1}^{J}\sum\limits_{t=K+1}^{n_j} \mathds{1}\{ I_t^j = k, k \neq k_*^{j},\\
&\hspace{0.5cm}N_k^{j}(t-1)<u_{1k}^{j},(N_k^{j}(t-1)+M_k^j)<u_{2k}^{j}\}
\end{align*}}
\vspace{-0.5cm}
{\begin{align*}
&\hspace{-0.5cm}= \sum\limits_{j=1}^{J}\sum\limits_{t=K+1}^{n_j} \min\bigg\{\mathds{1}\{ I_t^j = k, k \neq k_*^{t},N_k^{j}(t-1)<u_{1k}^{j}\},\\
&\hspace{0.5cm}\mathds{1}\{ I_t^j = k, k \neq k_*^{t},
(N_k^{j}(t-1)+M_k^j)<u_{2k}^{j}\}\bigg\},
\end{align*}}
\vspace{-0.25cm}
{\begin{align*}
&\hspace{0.1cm}\leq \sum\limits_{l=0}^{\lceil\frac{J-2}{2}\rceil}\min\Bigg\{\sum\limits_{j=\{2l+1,2l+2\}}\sum\limits_{t=K+1}^{n_j} \mathds{1}\{ I_t^j = k, k \neq k_*^{t},\\
&\hspace{0.5cm}N_k^{j}(t-1)<u_{1k}^{j}\},\sum\limits_{j=\{2l+1,2l+2\}}\sum\limits_{t=K+1}^{n_j}\mathds{1}\{ I_t^j = k, k \neq k_*^{t},\\
&\hspace{0.5cm}(N_k^{j}(t-1)+M_k^j)<u_{2k}^{j}\}\Bigg\}+(\mathds{1}\{J \: \text{is odd}\})\\
&\hspace{0.5cm}\Bigg(\min\{\sum\limits_{t=K+1}^{n_j} \mathds{1}\{ I_t^J = k, k \neq k_*^{t},N_k^{J}(t-1)<u_{1k}^{J}\},\\
&\hspace{0.5cm}\sum\limits_{t=K+1}^{n_j} \mathds{1}\{ I_t^J = k, k \neq k_*^{t},N_k^{J}(t-1)+M_k^J<u_{2k}^{J}\}\}\Bigg)
\end{align*}}
\vspace{-0.5cm}
{\begin{align*}
&\hspace{0.2cm}\leq \sum\limits_{l=0}^{\lceil\frac{J-2}{2}\rceil}\min\Bigg\{\sum\limits_{\substack{j=\{2l+1,2l+2\} \\
\Delta_{k}^{j}> 0}}\hspace{-0.2cm}u_{1k}^j,\sum\limits_{j=\{2l+1,2l+2\}}\sum\limits_{t=K+1}^{n_j}\hspace{-0.1cm}\mathds{1}\{ I_t^j = k,\\
&\hspace{0.5cm}k \neq k_*^{t},(N_k^{j}(t-1)+\hat{M}_k^j)<\max\{u_{2k}^{2l+1},u_{2k}^{2l+2}\}\}\Bigg\}+\\
&\hspace{0.5cm}\mathds{1}\{J \: \text{is odd},\Delta_k^j>0\}\Bigg(\min\bigg\{u_{1k}^J, u_{2k}^J\bigg\}\Bigg)
\end{align*}}
where $\hat{M}_k^j = M_k^j\mathds{1}\{j\: \text{is even}\}$,
{\begin{align}
&\leq \sum\limits_{l=0}^{\lceil\frac{J-2}{2}\rceil}\min\Bigg\{\underbrace{\sum\limits_{\substack{j=\{2l+1,2l+2\} \\
\Delta_{k}^{j}> 0}}u_{1k}^j}_{\triangleq U_k^l},\nonumber\\
&\hspace{0.5cm}\underbrace{\Bigg(\sum\limits_{\substack{j=\{2l+1,2l+2\} \\
\Delta_{k}^{j}> 0}}u_{2k}^j\Bigg)-\min\{\max\{u_{2k}^{2l+1},u_{2k}^{2l+2}\},B_k\}}_{\triangleq V_k^l}\Bigg\}+\nonumber\\
&\hspace{0.5cm}\underbrace{\mathds{1}\{J \: \text{is odd},\Delta_k^j>0\}\Bigg(\min\bigg\{u_{1k}^J, u_{2k}^J\bigg\}\Bigg)}_{\triangleq W_k^J}
\end{align}}
\vspace{-0.5cm}
{\begin{align}
\label{eq:S_k_second}
&\text{Second term of \eqref{eq:S_k}}\leq \sum\limits_{j=1}^{J}\bigg(1+\sum\limits_{t=K+1}^{n_j} \mathds{1}\{\eqref{eq:cond1}\hspace{0.1cm} \text{or}\hspace{0.1cm} \eqref{eq:cond2}\text{or}\hspace{0.1cm}\eqref{eq:cond3}\nonumber\\
&\hspace{3.5cm}\text{or}\hspace{0.1cm} \eqref{eq:cond4}\hspace{0.1cm} \text{is}\hspace{0.1cm} \text{True}\}\bigg).
\end{align}}
Using \eqref{eq:S_k}, \eqref{eq:S_k_first}, \eqref{eq:S_k_second} and taking expectation, we obtain
{\begin{align}
\label{eq:S_k_f+S2}
\mathbb{E}[D_k(J)]&\leq \sum\limits_{l=0}^{\lceil\frac{J-2}{2}\rceil}\min\left\{U_{k}^l,V_{k}^l\right\}+W_k^J+ \sum\limits_{j=1}^{J}\bigg(1+\nonumber\\
&\hspace{0.5cm}\sum\limits_{t=K+1}^{n_j} \text{Pr}\{\eqref{eq:cond1}\hspace{0.1cm}\text{or} \eqref{eq:cond2}\hspace{0.1cm} \text{or}\hspace{0.1cm} \eqref{eq:cond3}\hspace{0.1cm} \text{or}\hspace{0.1cm} \eqref{eq:cond4}\hspace{0.1cm} \text{is}\hspace{0.1cm} \text{True}\}\bigg),\\
\nonumber
\end{align}}
Next, we bound the probability of the event that at least one of \eqref{eq:cond1} or \eqref{eq:cond2} or \eqref{eq:cond3} or \eqref{eq:cond4} is true. We use the union bound, followed by the application of one-sided Hoeffding's inequality \cite{hoeffding1994probability} to get,
{\begin{align}
\label{eq:probBound2}
&\text{Pr}\{\eqref{eq:cond1}\hspace{0.1cm}\text{or}\hspace{0.1cm} \eqref{eq:cond2}\hspace{0.1cm} \text{or}
\hspace{0.1cm} \eqref{eq:cond3}\hspace{0.1cm} \text{or}\hspace{0.1cm}  \eqref{eq:cond4}\hspace{0.1cm} \text{is}\hspace{0.1cm} \text{True}\}\nonumber\\
&\leq \text{Pr}\{\eqref{eq:cond1}\hspace{0.1cm} \text{is}\hspace{0.1cm} \text{True}\}+\text{Pr}\{\eqref{eq:cond2}\hspace{0.1cm} \text{is}\hspace{0.1cm} \text{True}\}+\text{Pr}\{\eqref{eq:cond3}\hspace{0.1cm} \text{is}\hspace{0.1cm} \text{True}\}\nonumber\\
&\hspace{1cm}+\text{Pr}\{\eqref{eq:cond4}\hspace{0.1cm} \text{is}\hspace{0.1cm} \text{True}\}\}\nonumber\\
&\leq \frac{2}{t^{\alpha-1}}+2\text{Pr}\{\eqref{eq:cond3}\hspace{0.1cm} \text{is}\hspace{0.1cm} \text{True}\}
\end{align}}
\begin{align*}
\text{Pr}\{\eqref{eq:cond3}\hspace{0.1cm} \text{is}\hspace{0.1cm} \text{True}\}&= \text{Pr}\{\hat{\mu}_{2} - q_2  > \mu\}\\
&=\text{Pr}\bigg\{\hat{\mu}_{2} - \mu  > \sqrt{\frac{\eta\log{(B_k+t-1)}}{2(N_{k}^j(t)+M_k^j)}}\bigg\}
\end{align*}
\begin{align}
\label{proof_bias}
=\text{Pr}\bigg\{\hat{\mu}_{2} - \mathbb{E}[\hat{\mu}_{2}] > \sqrt{\frac{\eta\log{(B_k+t-1)}}{2(N_{k}^j(t)+M_k^j)}}-\frac{M_k^j\epsilon_k}{(N_{k}^j(t)+M_k^j)}\bigg\}
\end{align}
We want the bias term to not exceed more than half of the confidence width, 
\begin{align*}
\frac{M_k^j\epsilon_k}{(N_{k}^j(t)+M_k^j)}&\leq \frac{1}{2}\sqrt{\frac{\eta\log{(B_k+t-1)}}{2(N_{k}^j(t)+M_k^j)}}\\
B_k&\leq \frac{\eta-4\epsilon_k^2}{4\epsilon_k^2},
\end{align*}
which is ensured by the algorithm.
Therefore, from \ref{proof_bias} we get, 
\begin{align}
\label{proof_probbound}
\text{Pr}\{\eqref{eq:cond3}\hspace{0.1cm} \text{is}\hspace{0.1cm} \text{True}\}&=\text{Pr}\bigg\{\hat{\mu}_{2} - \mathbb{E}[\hat{\mu}_{2}] > \sqrt{\frac{\eta\log{(B_k+t-1)}}{2(N_{k}^j(t)+M_k^j)}}-\nonumber\\
&\hspace{1cm}\frac{M_k^j\epsilon_k}{(N_{k}^j(t)+M_k^j)}\bigg\}\nonumber\\
&\leq \text{Pr}\bigg\{\hat{\mu}_{2} - \mathbb{E}[\hat{\mu}_{2}] > \frac{1}{2}\sqrt{\frac{\eta\log{(B_k+t-1)}}{2(N_{k}^j(t)+M_k^j)}}\bigg\}\nonumber\\
&\leq \frac{2}{(B_k+t)^{\frac{\eta}{4}-1}},
\end{align}
where the last inequality follows from Hoeffding's inequality.
Using \eqref{eq:S_k_f+S2}, \eqref{eq:probBound2} and \eqref{proof_probbound}, we obtain
{\begin{align*}
&\mathbb{E}[D_k(J)]\leq \sum\limits_{l=0}^{\lceil\frac{J-2}{2}\rceil}\min\left\{U_{k}^l,V_{k}^l\right\}+W_k^J
+\sum\limits_{j=1}^{J}\bigg(1+\\
&\hspace{2cm}\sum\limits_{t=K+1}^{n_j} \frac{2}{t^{\alpha-1}}+\frac{2}{(t+B_k)^{\alpha-1}}\bigg)
\end{align*}}
{\begin{align*}
&\hspace{0cm}\leq \sum\limits_{l=0}^{\lceil\frac{J-2}{2}\rceil}\min\left\{U_{k}^l,V_{k}^l\right\}+W_k^J
+\sum\limits_{j=1}^{J}\bigg(1+\int\limits_{s=K}^{\infty} \bigg(\frac{2}{s^{\alpha-1}}+\\
&\hspace{2cm}\frac{2}{(s+B_k)^{\alpha-1}}\bigg)ds\bigg)
\end{align*}}
\vspace{-0.5cm}
{\begin{align*}
&\hspace{-0.6cm}= \sum\limits_{l=0}^{\lceil\frac{J-2}{2}\rceil}\min\left\{U_{k}^l,V_{k}^l\right\}+W_k^J+ \sum\limits_{j=1}^{J}\bigg(1+\frac{2K^{2-\alpha}}{\alpha-2}+\\
&\hspace{2cm}\frac{2(K+B_k)^{2-\frac{\eta}{4}}}{\frac{\eta}{4}-2}\bigg)
\end{align*}}
\vspace{-0.5cm}
{\begin{align*}
&\hspace{0.4cm}\leq \sum\limits_{l=0}^{\lceil\frac{J-2}{2}\rceil}\min\left\{U_{k}^l,V_{k}^l\right\}+W_k^J+ \sum\limits_{j=1}^{J}\bigg(1+\frac{2}{\alpha-2}+\frac{2}{\frac{\eta}{4}-2}\bigg)
\end{align*}}
\vspace{-0.5cm}
{\begin{align*}
&\hspace{-0.3cm}\leq \sum\limits_{l=0}^{\lceil\frac{J-2}{2}\rceil}\min\left\{U_{k}^l,V_{k}^l\right\}+W_k^J+ J\bigg(\frac{\alpha}{\alpha-2}+\frac{8}{\eta-8}\bigg),
\end{align*}}
and therefore \eqref{eq:regretExp_2} follows.

\subsection{Proof of Theorem \ref{theorem2_trucb}}
\label{proof:theorem_3}
Let $\hat{B}_k$ and $\hat{u}_{2k}^{j}$ be defined as follows,
\begin{align}
\label{eq_B_hat_u2_hat}
\hat{B}_k \triangleq \frac{\eta-4\hat{\epsilon}_k^2}{4\hat{\epsilon}_k^2},\:\:\text{and}\:\: 
\hat{u}_{2k}^{j} \triangleq \frac{2\eta\log{(\hat{B}_k+n_j)}}{(\Delta_k^{j})^2}.
\end{align}
For arm $k$ to be pulled at time $t$ in task $j$ (i.e. $I_t^j = k$), at least one of the following five conditions should be true:
\begin{align}
\hat{\mu}_{1} - q_1 &> \mu \label{eq:cond1_2} \\
\hat{\mu}_{1*} + q_{1*} &\leq \mu_{*}
\label{eq:cond2_2}\\
\hat{\mu}_{2*} + \hat{q}_{2*} &\leq \mu_{*} \label{eq:cond3_2}\\
\hat{\mu}_{2} - \hat{q}_2 & > \mu \label{eq:cond4_2}
\end{align}
\vspace{-0.8cm}
\begin{align}
N_k^{j}(t-1)  < u_{1k}^{j}\hspace{0.1cm}\text{and}\hspace{0.1cm}\
N_k^{j}(t-1)+\hat{M}_k^j <  \hat{u}_{2k}^{j} \label{eq:cond5_1}
\end{align}
This is shown by contradiction. The steps are similar as in the proof of Theorem \ref{proof:theorem_2} and therefore is omitted here. The total regret after total tasks $J$ is given by
{\begin{align}
R_J &= \sum\limits_{j=1}^{J}\sum\limits_{k=1}^K \Delta_{k}^{j}\mathbb{E}[N_k^j (n_j)]=\sum\limits_{k=1}^{K}\Delta_k^{\max}\mathbb{E}[D_k (J)],
\end{align}}
where $D_k(J)$ is the total number of sub-optimal pulls to arm $k$ over all tasks. Next, we bound the regret by bounding the term $\mathbb{E}[D_k(J)]$. For an arbitrary sequence $I_t^j$, $t=1,2,\cdots,n_j$, $\forall j\in[J]$, we have 
{\begin{align}
\label{eq:S_k2}
&D_k(J)=\sum\limits_{j=1}^{J}\sum\limits_{t=1}^{n_j}\mathds{1}\{ I_t^j = k,k\neq k_{*}^j\},\nonumber\\
&=\sum\limits_{j=1}^{L}\bigg(\frac{l}{K}\mathds{1}\{k\neq k_{*}^j\}+\sum\limits_{t=l+1}^{n_j}\mathds{1}\{ I_t^j = k,k\neq k_{*}^j\}\bigg)+\nonumber\\
&\hspace{1cm}\sum\limits_{j=L+1}^{J}\bigg(\mathds{1}\{k\neq k_{*}^j\}+\sum\limits_{t=K+1}^{n_j}\mathds{1}\{ I_t^j = k,k\neq k_{*}^j\}\bigg),\nonumber\\
&\leq \frac{l L}{K}+\sum\limits_{j=1}^{L}\sum\limits_{t=l+1}^{n_j}\mathds{1}\{ I_t^j = k,k\neq k_{*}^j\}+\sum\limits_{j=L+1}^{J}\bigg(\mathds{1}\{k\neq k_{*}^j\}+\nonumber\\
&\hspace{1cm}\sum\limits_{t=K+1}^{n_j}\mathds{1}\{ I_t^j = k,k\neq k_{*}^j\}\bigg).
\end{align}}
We omit the steps showing the further simplification of \eqref{eq:S_k2} to obtain an upper bound on the expected number of sub-optimal pulls since it is similar to the proof of Theorem \ref{proof:theorem_2}. Therefore, we get 
{\begin{align*}
\mathbb{E}[D_k(J)]&\leq \frac{l L}{K}+\sum\limits_{l=0}^{\lceil\frac{J-2}{2}\rceil}\min\left\{U_{k}^l,\hat{V}_{k}^l\right\}+\hat{W}_k^J+ \sum\limits_{j=1}^{J}\bigg(1+\\
&\hspace{0.5cm}\sum\limits_{t=K+1}^{n_j} \text{Pr}\{\eqref{eq:cond1_2}\hspace{0.1cm}\text{or}\hspace{0.1cm} \eqref{eq:cond2_2}
\hspace{0.1cm}\text{or}\hspace{0.1cm} \eqref{eq:cond3_2}\hspace{0.1cm} \text{or}\hspace{0.1cm} \eqref{eq:cond4_2}\hspace{0.1cm} \text{is}\hspace{0.1cm} \text{True}\}\bigg),
\end{align*}}
where,
{\begin{align*}
U_{k}^l = u_{1k}^{2l+1}\mathds{1}\{\Delta_{k}^{2l+1}> 0\}+u_{1k}^{2l+2}\mathds{1}\{\Delta_{k}^{2l+2}> 0\},
\end{align*}}
{\begin{align*}
\hat{V}_{k}^l=&\Bigg(\hat{u}_{2k}^{2l+1}\mathds{1}\{\Delta_{k}^{2l+1}> 0\}+\hat{u}_{2k}^{2l+2}\mathds{1}\{\Delta_{k}^{2l+2}> 0\}\Bigg)\\
&\hspace{0.5cm}-\min\Bigg\{\max\bigg\{\hat{u}_{2k}^{2l+1},\hat{u}_{2k}^{2l+2}\bigg\},\hat{B}_k\Bigg\}
\end{align*}}
{\begin{align*}
\hspace{0cm}\text{and}\:\:\hat{W}_{k}^J= \mathds{1}\{J \: \text{is odd},\Delta_k^J>0\}\Bigg(\min\bigg\{u_{1k}^J, \hat{u}_{2k}^J\bigg\}\Bigg).
\end{align*}}
Further, using $\min\left\{U_{k}^l,\hat{V}_{k}^l\right\}\leq U_{k}^l$ and $\min\bigg\{u_{1k}^J, \hat{u}_{2k}^J\bigg\}\leq u_{1k}^J$, we get
{\begin{align*}
\mathbb{E}[D_k(J)]&\leq \frac{l L}{K}+\sum\limits_{l=0}^{\lceil\frac{J-2}{2}\rceil}U_{k}^l+u_{1k}^J\mathds{1}\{J \: \text{is odd},\Delta_k^J>0\}\\
&\hspace{0.5cm}+ \sum\limits_{j=1}^{J}\bigg(1+\sum\limits_{t=K+1}^{n_j} \text{Pr}\{\eqref{eq:cond1_2}\hspace{0.1cm}\text{or}\hspace{0.1cm} \eqref{eq:cond2_2}\hspace{0.1cm}\text{or}\hspace{0.1cm} \eqref{eq:cond3_2}\hspace{0.1cm} \\
&\hspace{0.8cm}\text{or}\hspace{0.1cm} \eqref{eq:cond4_2}\hspace{0.1cm} \text{is}\hspace{0.1cm} \text{True}\}\bigg).
\end{align*}}
{\begin{align}
&= \frac{l L}{K}+\sum\limits_{j=1}^{J}u_{1k}^j+ \sum\limits_{j=1}^{J}\bigg(1+\sum\limits_{t=K+1}^{n_j} \text{Pr}\{\eqref{eq:cond1_2}\hspace{0.1cm}\text{or}\hspace{0.1cm} \eqref{eq:cond2_2}\hspace{0.1cm}\text{or}\hspace{0.1cm} \eqref{eq:cond3_2}\nonumber\\
\label{eq:S_k_f+S_B}
&\hspace{0.5cm}\text{or}\hspace{0.1cm} \eqref{eq:cond4_2}\hspace{0.1cm} \text{is}\hspace{0.1cm} \text{True}\}\bigg),
\end{align}}

Next, we bound the probability of the event that at least one of \eqref{eq:cond1_2} or \eqref{eq:cond2_2} or \eqref{eq:cond3_2} or \eqref{eq:cond4_2} is true. We use the union bound, followed by the application of one-sided Hoeffding's inequality \cite{hoeffding1994probability} to obtain,
{\begin{align}
\label{eq:probBound}
&\text{Pr}\{\eqref{eq:cond1_2}\hspace{0.1cm}\text{or}\hspace{0.1cm} \eqref{eq:cond2_2}\hspace{0.1cm} \text{or}
\hspace{0.1cm} \eqref{eq:cond3_2}\hspace{0.1cm} \text{or}\hspace{0.1cm}  \eqref{eq:cond4_2}\hspace{0.1cm} \text{is}\hspace{0.1cm} \text{True}\}\nonumber\\
&\leq \text{Pr}\{\eqref{eq:cond1_2}\hspace{0.1cm} \text{is}\hspace{0.1cm} \text{True}\}+\text{Pr}\{\eqref{eq:cond2_2}\hspace{0.1cm} \text{is}\hspace{0.1cm} \text{True}\}+\text{Pr}\{\eqref{eq:cond3_2}\hspace{0.1cm} \text{is}\hspace{0.1cm} \text{True}\}+\nonumber\\
&\hspace{1cm}\text{Pr}\{\eqref{eq:cond4_2}\hspace{0.1cm} \text{is}\hspace{0.1cm} \text{True}\}\},\nonumber\\
&\leq \frac{2}{t^{\alpha-1}}+\text{Pr}\{\eqref{eq:cond3_2}\hspace{0.1cm} \text{is}\hspace{0.1cm} \text{True}\}+\text{Pr}\{\eqref{eq:cond4_2}\hspace{0.1cm} \text{is}\hspace{0.1cm} \text{True}\}\nonumber\\
&\leq \frac{2}{t^{\alpha-1}}+\text{Pr}\{\hat{\mu}_{2*} + q_{2*} \leq \mu_{*}\}+ \text{Pr}\{ \hat{\epsilon}_k^j<\epsilon_k\}+\nonumber\\
&\hspace{1cm}\text{Pr}\{\hat{\mu}_{2} - q_2  > \mu\}+ \text{Pr}\{ \hat{\epsilon}_k^j<\epsilon_k\}\nonumber\\
&\leq \frac{2}{t^{\alpha-1}}+2\bigg(\frac{1}{(t+B_k)^{\frac{\eta}{4}-1}}+\text{Pr}\{ \hat{\epsilon}_k^j<\epsilon_k\}\bigg)\nonumber\\
&\leq \frac{2}{t^{\alpha-1}}+2\bigg(\frac{1}{(t+B_k)^{\frac{\eta}{4}-1}}+j\delta\bigg).
\end{align}}
Using \eqref{eq:S_k_f+S_B} and \eqref{eq:probBound}, we get
{\begin{align*}
&\mathbb{E}[D_k(J)]\leq \frac{l L}{K}+\sum\limits_{j=1}^{J}u_{1k}^j
+\nonumber\\
&\hspace{2cm}\sum\limits_{j=1}^{J}\bigg(1+\sum\limits_{t=K+1}^{n_j} \frac{2}{t^{\alpha-1}}+\frac{2}{(t+B_k)^{\frac{\eta}{4}-1}}+ j\delta\bigg),\\
&\leq \frac{l L}{K}+\sum\limits_{j=1}^{J}u_{1k}^j+J\bigg(\frac{\alpha}{\alpha-2}+\frac{8}{\eta-8}\bigg)+TJ\delta,
\end{align*}}
and therefore \eqref{eq:trucbRegret2} follows.

For the next part, lets consider again \eqref{eq:trucbRegret2},
{\begin{align*}
R_T&\leq \frac{l L}{K}+\sum\limits_{j=1}^{J}u_{1k}^j+ J\bigg(\frac{\alpha}{\alpha-2}+\frac{8}{\eta-8}\bigg)+TJ\delta\\
&=\mathcal{O}(T^\beta\log(T))+\mathcal{O}(T^\beta)+\mathcal{O}(T^\beta)\\
&=\mathcal{O}(T^\beta\log(T))
\end{align*}}
Hence the theorem follows.

%%%%%%%%%%%%%%%%%%%%%%%%%%%%%%%%%%%%%%%%%%%%%%%%%%%%%%%%%%%%%%%%%%%%%%%%%%%%%%%%

\bibliographystyle{ieeetr}
\bibliography{reference}

\end{document}